\documentclass{article}

\usepackage{fullpage}
\usepackage{url}            % simple URL typesetting
\usepackage{booktabs}       % professional-quality tables
\usepackage{amsfonts}       % blackboard math symbols
\usepackage{nicefrac}       % compact symbols for 1/2, etc.
\usepackage{microtype}      % microtypography
\usepackage{amsthm}
\usepackage{amsmath}
\usepackage{amssymb}
\usepackage[dvipsnames,usenames]{xcolor}

\newtheorem{definition}{Definition}
\newtheorem{assumption}{Assumption}
\newtheorem{theorem}{Theorem}
\newtheorem{lemma}{Lemma}
\newtheorem{corollary}{Corollary}

\usepackage{algorithm}
\usepackage{algorithmic}
\usepackage{color}
\usepackage{prettyref}
\usepackage{units}
\newcommand{\cvm}{{V}}
\newcommand{\cvS}{{v}}
\newcommand{\cost}{\cv}
\newcommand{\cv}{{\bf \cvS}}
\newcommand{\vt}{\tilde{\cv}_t}
\newcommand{\rt}{\tilde{r}_t}

\newcommand{\rv}{{\bf r}}

\newcommand{\cvA}{{\bf \cvS}}
\newcommand{\cvE}{{\hat{\cv}}}

\newcommand{\Ex}{{\mathbb E}}

\newcommand{\aregO}{\text{avg-regret}^1}
\newcommand{\aregD}{\text{avg-regret}^2}
\newcommand{\regret}{\text{regret}}

\newcommand{\linCBwK}{\text{linCBwK}\enspace}

\newcommand{\linCBwC}{\text{linCBwC}}
\newcommand{\linCbudget}{\text{linCBwK}}
\newcommand{\OCO}{\text{OCO}\enspace}
\newcommand{\thetaV}{{\boldsymbol{\theta}}}

\newcommand{\di}{d}

\newcommand{\comment}[1]{}
\newcommand{\ShipraNIPS}[1]{\textcolor{red}{#1}}

\newcommand{\oneNorm}{\|{\bf{1}}_\di\|}

\newcommand{\cx}{{X}}
\newcommand{\cxL}{{\bf x}}
\newcommand{\Mean}{W}
\newcommand{\Meanstar}{\Mean_*}
\newcommand{\MeanEst}{\hat{\Mean}}
\newcommand{\MeanUCB}{\tilde{\Mean}}
\newcommand{\MeanVector}{{\boldsymbol{\mu}}}
\newcommand{\prob}{\boldsymbol{p}}

\newcommand{\OPT}{\text{OPT}}
\newcommand{\OPTest}{\hat{\text{OPT}}}

\newcommand{\x}{\bs{x}}
\newcommand{\y}{\bs{y}}	
\newcommand{\z}{\bs{z}}

\newcommand{\regOCO}{{\cal R}}
\newcommand{\reglinCBwK}{{\cal R'}}
\newcommand{\allcontexts}{\mathcal{X}}
\newcommand{\cvavg}[1]{\bar{\cv}_{1:#1}}

\newcommand{\vsum}{\tilde{\cv}_{\text{sum}}}

\newcommand{\rsum}{\tilde{r}_{\text{sum}}}

\newcommand{\ones}{{\bf 1}}

\newcommand{\Real}{\mathbb{R}}

\newcommand{\OMD}{OMD\enspace}
\newcommand{\bs}[1]{\boldsymbol{#1}}
\newcommand{\domainTheta}{\Omega}

\newcommand{\domain}{[0,1]}

\newcommand{\gammaValue}{2  m \sqrt{T_0 {\ln(T_0) \ln(T_0d/\delta)}}}

\newcommand{\shipra}[1]{[\textcolor{red}{SA}: #1]}
\newcommand{\nikhil}[1]{[\textcolor{red}{ND}: #1]}

\newcommand{\OPTinterim}{\overline{\text{OPT}}}
\newcommand{\rmu}{\boldsymbol{\mu}}

\newcommand{\SSPi}{{\cal C}_0(\Pi)}

\newcommand{\1}[1]{\mathbb{I}\left\{#1\right\}}

\newcommand{\eqnref}[1]{Equation~\eqref{#1}}

\newcommand{\opthat}{{\hat{\text{OPT}}}}
\newcommand{\opthatgammat}{\opthat^\gamma_t}
\newcommand{\rhat}{\hat{r}}

\newcommand{\opt}{\OPT}
\newcommand{\ropt}{r^*}
\newcommand{\vopt}{\cvE^*}

\title{Linear Contextual Bandits with Knapsacks}

\author{Shipra Agrawal\thanks{Columbia University. \tt{sa3305@columbia.edu}.}
	\and Nikhil R. Devanur\thanks{Microsoft Research. \tt{nikdev@microsoft.com}.}
}
\begin{document}
\date{}
\maketitle
\begin{abstract}
We consider the linear contextual bandit problem with resource consumption, in addition to reward generation. In each round, the outcome of pulling an arm is a reward as well as a vector of resource consumptions. The expected values of these outcomes depend linearly on the context of that arm. The budget/capacity constraints require that the total consumption doesn't exceed the budget for each resource. The objective is once again to maximize the total reward. This problem turns out to be a common generalization of classic linear contextual bandits  (linContextual) \cite{Auer2003, chu2011, oful},  bandits with knapsacks (BwK) \cite{AD14,BwK}, and the online stochastic packing problem (OSPP)  \cite{AD15,Devanur2011}. We present algorithms with near-optimal regret bounds for this problem. 
Our bounds compare favorably to results on the unstructured version of the problem \cite{ADL16,Badanidiyuru14Resourceful} 
where the relation between the contexts and the outcomes could be arbitrary, but the algorithm only competes against a fixed set of policies accessible through  an optimization oracle. 
%Even though linear contextual bandits is a special case of contextual bandits with general policy space, we show that when global constraints are introduced the policy space for linear bandits becomes doubly exponential in its dimensions, so that the techniques and bounds for the general contextual bandits do not apply. 
We combine techniques from the work on linContextual, BwK and OSPP in a nontrivial manner while also tackling new difficulties 
that are not present in any of these special cases.
\end{abstract}

%\thispagestyle{empty}
%\newpage
%\setcounter{page}{1}

\section{Introduction}
\label{sec:intro}
%Contextual bandits is a very general framework for sequential decision making under uncertainty. 
In the contextual bandit problem \cite{Auer2003, Beygelzimer2011, Monster,MiniMonster}, the decision maker observes a sequence of contexts (or features). In every round she needs to pull one out of $K$ arms, after observing the context for that round. The outcome of pulling an arm may be used along with the contexts to decide future arms. Contextual bandit problems have found many useful applications such as online recommendation systems, online advertising, and clinical trials, where the decision in every round needs to be customized to the features of the user being served. 
The {\em linear contextual bandit} problem \cite{oful, Auer2003, chu2011} is a special case of the contextual bandit problem, where the outcome is linear in the feature vector encoding the context. As pointed by \cite{MiniMonster}, contextual bandit problems represent a natural half-way point between supervised learning and reinforcement learning: the use of features to encode contexts and  the models for the relation between these feature vectors and the outcome are often inherited from supervised learning, while managing the exploration-exploitation tradeoff is necessary to ensure good performance in reinforcement learning.  The linear contextual bandit problem can thus be thought of as a midway between the linear regression model of supervised learning, and reinforcement learning.

Recently, there has been a significant interest in introducing multiple ``global constraints" in the standard bandit setting \cite{BwK, AD14,Badanidiyuru14Resourceful,ADL16}. 
%One of the major limitations of the standard bandit setting is the lack of ``global'' constraints that 
Such constraints are crucial for many important real-world applications.  For example, in clinical trials, the treatment plans may be constrained by the total availability of medical facilities, drugs and other resources.  In online advertising, there are budget constraints that restrict the number of times an ad is shown. 
Other applications include dynamic pricing, dynamic procurement, crowdsourcing, etc.; 
see \cite{BwK, AD14} for many such examples. 
%In recommendation systems, the recommendations made to individual users need to be served from a limited common inventory of goods.
%For example, actions taken by a robot arm may have different levels of power consumption, and the total power consumed by the arm is limited by the capacity of its battery.  In online advertising, each advertiser has her own budget, so that her advertisement cannot be shown more than a certain number of times. In dynamic pricing, there are a certain number of objects for sale and the seller offers prices to a sequence of buyers with the goal of maximizing revenue, but the number of sales is limited by the supply. 

In this paper, we consider {\bf linear contextual bandit with knapsacks} (henceforth, \linCbudget) problem. In this problem, the context vectors are generated i.i.d. in every round from some unknown distribution, and on picking an arm, a reward and {\it a consumption vector} is observed, which depend linearly on the context vector. The aim of the decision maker is to maximize a total reward while ensuring the the total consumption of every resource remains withing a given budget. Below, we give a more precise definition of this problem. We use the following  notational convention throughout: vectors are denoted by bold face lower case letters, while matrices are denoted by regular face upper case letters. Other quantities such as sets, scalars, etc. may be of either case, but never bold faced. All vectors are column vectors, i.e., a vector in $n$ dimensions is treated as an $n\times 1$ matrix.  The transpose of matrix $A$ is $A^\top$. 
%We use the terms ``arm" and ``action" interchangeably. 

\begin{definition}[\linCbudget] \label{def:lincbwcr}
	There are $K$ ``arms", which we identify with the set $[K]$. 
	The algorithm is initially given as input a budget $B \in \mathbb{R}_+$.
	%convex set $S\subseteq \domain^\di$, and an $L$-Lipschitz concave function $f$ with  domain $\domain^\di$. We then 
	%We proceed in rounds $t=1,2, \ldots,T,$ where in every round $t$, a tuple of context matrix, reward vector and consumption matrix $({\cx}_t, \rv_t, \cvm_t) \in \domain^{m\times K} \times \domain^{K} \times \domain^{\di\times K}$ are  drawn from an unknown distribution ${\cal D}$, independent of everything in previous rounds. The column of $\cx_t$ corresponding to arm $a\in [K]$ is called the context vector for arm $a$ and is denoted by $\cxL_t(a) \in \domain^m$. The column of $\cvm_t$ corresponding to an arm $a$ is called the ``outcome'' vector for arm $a$  and is denoted by $\cv_t(a)\in \domain^\di$. There exist an (unknown) weight matrix $\Meanstar \in \domain^{m \times \di}$ such that 
%$$\Ex[\cvm_t | \cx_t] =\Meanstar^\top  \cx_t.$$
 In every round $t$, the algorithm first observes context ${\cxL}_t(a) \in \domain^{m}$ for every arm $a$, and then chooses an arm $a_t \in [K]$, and finally observes a reward $r_t(a_t) \in [0,1]$ and a $d$-dimensional consumption vector $\cv_t(a_t) \in [0,1]^d$. The algorithm has a ``no-op" option, which is to pick none of the arms and get $0$ reward and ${\bf 0}$ consumption. The goal of the algorithm is to pick arms such that the total reward $\sum_{t=1}^T r_t(a_t)$ is maximized, while ensuring that the total consumption does not exceed budget, i.e., $\sum_t \cv_t(a_t) \le B\ones$. 

We make the following stochastic assumption for context, reward, consumption vectors. In every round $t$, the tuple $\{x_t(a), r_t(a), \cv_t(a)\}_{a=1}^K$  is generated from an unknown distribution ${\cal D}$, independent of everything in previous rounds. Also, there exists an unknown vector $\mu_* \in \domain^{m}$ and matrix $W_*\in \domain^{m \times \di}$ such that for every arm $a$, given contexts $x_t(a)$, and history $H_{t-1}$ before time $t$, 
\begin{equation}
\label{eq:linearAssumption}
\Ex[r_t(a)|x_t(a), H_{t-1}] = \mu_*^\top x_t(a), \ \ \  \Ex[\cv_t(a)|x_t(a), H_{t-1}] = W_*^\top x_t(a).
\end{equation}
%average vector $\frac{1}{T} \sum_{t=1}^T \cv_t(a_t)$ maximizes $f(\frac{1}{T} \sum_{t=1}^T \cv_t(a_t))$ and is inside $S$. 
For succinctness, we will denote the tuple of contexts for $K$ arms at time $t$ as matrix $X_t \in \domain^{m \times K}$, with $\cxL_t(a)$ being the $a^{th}$ column of this matrix. Similarly, rewards are represented as vector $\rv_t \in \domain^{K}$, and consumption vectors are represented as matrix $V_t \in \domain^{\di\times K}$.
\end{definition}

As we discuss later in the text, the assumption in equation \eqref{eq:linearAssumption} forms the primary distinction between our linear contextual bandit setting and the general contextual bandit setting considered in \cite{ADL16}. Exploiting this linearity assumption will allow us to generate regret bounds which do not depend on number of arms $K$, rendering it to be especially useful when number of arms is large. Some examples include recommendation systems with large number of products (e.g., retail products, travel packages, ad creatives, sponsored facebook posts). Another advantage over using 
general contextual bandit setting of \cite{ADL16} is that we don't need an oracle access to a certain optimization problem, which is required to solve an NP-Hard problem in this case. (See Section \ref{sec:mainresults} for a more detailed discusssion.)

We compare the performance of an algorithm to that of the optimal adaptive policy that knows the distribution ${\cal D}$ and the parameters $(\mu_*, \Mean_*)$, and can take into account the history upto that point as well as the current context to decide (possibly with randomization) which arm to pull at time $t$. 
%We compare the performance of an algorithm to that of an optimal adaptive policy that  takes into account  the distribution ${\cal D}$, the history upto that point and the current context to decide (possibly with randomization) which arm to pull. 
However, it is easier to work with an upper bound on this, which is the optimal expected reward of a static policy that is required to satisfy the constraints only in expectation. 
This technique has been used in several related problems and is  standard by now \cite{Devanur2011,BwK}.

\begin{definition}[Optimal Static Policy]
Consider any policy that is context dependent but non-adaptive: for a policy $\pi$, let $\pi(\cx)\in \Delta^{K+1}$ (the unit simplex) denote the probability distribution over arms played (plus no-op) when the context is $\cx \in \allcontexts$. Define $\rv(\pi)$ and $\cvA(\pi)$ to be the expected reward and consumption vector of policy $\pi$, respectively, i.e.
\begin{eqnarray}
\rv(\pi) & :=& \Ex_{(\cx,\rv, \cvm) \sim {\cal D}}[ \rv \pi(\cx) ] = \Ex_{\cx \sim {\cal D}}[ \mu_*^\top \cx \pi(\cx)].\\
\cvA(\pi) & :=& \Ex_{(\cx,\rv, \cvm) \sim {\cal D}}[ \cvm \pi(\cx) ] = \Ex_{\cx \sim {\cal D}}[ \Meanstar^\top \cx \pi(\cx)].
\end{eqnarray}
\begin{eqnarray}
\text{Let\ \ \ }\pi^* & := & \begin{array}{lll}
 \arg\max_{\pi} & T\ \rv(\pi) ~~\text{such that} & T\ \cvA(\pi) \le B\ones
\end{array} \label{eq:optProb}
\end{eqnarray}
be the optimal static policy. Note that since no-op is allowed, a feasible policy always exists.
%We assume that there exists a feasible policy, i.e. $\cvA(\pi^*) \in S$,
%Let $\cx^*_j=\Ex[\sum_a \cx_{j}(a) q^*(\cx, a)]$. 
We denote the value of this optimal static policy by 
$\OPT:=T\ \rv(\pi^*).$
\end{definition}

%\ShipraNIPS{Add a theorem stating this is an upper bound on optimal adaptive policy's value}
Following lemma proves that $\OPT$ upper bounds the value of optimal {\it adaptive} policy. The proof is in Appendix \ref{app:benchmark}. 
\begin{lemma}\label{lem:staticOPT}
Let $\overline{\OPT}$ denote the value of optimal adaptive policy that knows the distribution ${\cal D}$ and parameters $\MeanVector_*, \Mean_*$,  We show that there exists a static policy $\pi^*$ such that $T \rv(\pi^*) \ge \overline{\OPT}$, and $T \cvA(\pi^*)\le  B$.
\end{lemma}

%$\prob^*$. Let $\OPT = f(\sum_a \Ex[\cv(a)]\prob^*(a))$. Next, we observe that $\OPT$ is an upper bound on expected  objective value obtainable by any feasible online policy. 

\comment{
\begin{lemma}
For a given policy $\pi$, denote its probability to play arm $a$ at time as $\prob_t(a) = \pi(\cx)$. Suppose that the policy ensures that the decisions are feasible in expectation, i.e.,
$\Ex[\frac{1}{t} \sum_t \sum_a \cv_t(a) p_t(a)] \in S$. Then,
we have that $\OPT \ge \Ex[f(\frac{1}{t} \sum_t \sum_a \cv_t(a) p_t(a))]$. 
\end{lemma}
\begin{proof}
Compute $q(\cx, a)$ as the probability of playing arm $a$ given context  $\cx$, i.e., $q(\cx, a)=\Ex[p_t(a) | \cx_t=\cx]$. Then,
\begin{eqnarray*}
\Ex[ \frac{1}{t} \sum_t \sum_a \cv_t(a) p_t(a)] & = & \sum_a \Ex[ \frac{1}{t} \sum_t\cv_t(a) p_t(a)]] \\
& = & \sum_a \Ex[\cv(a)] \Ex[\frac{1}{t} \sum _t p_t(a)]\\
& = & \sum_a \Ex[\cv(a)] \hat{p}(a)\\
\end{eqnarray*}
where $\hat{p}(a) := \Ex[\frac{1}{t} \sum _t p_t(a)]$. And,  since $\Ex[\frac{1}{t} \sum_t \sum_a \cv_t(a) p_t(a)])\in S$, it means $\hat{\prob}$ is a feasible solution to the optimization problem \eqref{eq:optProb}. Now, since $\prob^*$ is the optimal solution to that problem,  we have that
\begin{eqnarray*}
\Ex[f(\frac{1}{t} \sum_t \cv_t(a_t) \prob_t(a_t))] & \le & f(\Ex[\frac{1}{t} \sum_t \cv_t(a_t) \prob_t(a_t)]) \\
& = & f(\sum_a \Ex[\cv(a)] \hat{\prob}(a))\\
& \le & f(\sum_a \Ex[\cv(a)] \prob^*(a))
\end{eqnarray*}
\end{proof}
}

\begin{definition}[Regret] 
%We define two regrets: regret in objective and regret in constraints. 
Let $a_t$ be the arm played at time $t$ by the algorithm.  %For convenience of notation let us define $\rvavg{T}:=\frac{1}{T} \sum_{t=1}^T \rv_t(a_t)$.
Then, regret is defined as 
$$\regret(T) := \OPT-\sum_{t=1}^T \rv_t(a_t)$$
%\begin{itemize}
%\item Regret in objective. $ \aregO(T) = \OPT - f(\cvavg{T}).$
%\item Regret in constraints is defined as the distance of the average of played vectors from $S$; we let $d(\cdot, \cdot)$ denote a distance function. 
%$\aregD(T) := d(\cvavg{T}, S).$
%\end{itemize}
\end{definition}
\subsection{Main results }\label{sec:mainresults}
Our main result is an algorithm with near-optimal regret bound for \linCBwK.  
\begin{theorem}
\label{th:general}  There is an algorithm for \linCBwK such that 
if $B>m T^{3/4}$, then with probability at least $1-\delta,$   
\[ \regret(T) = O\left( (\tfrac \OPT B + 1){m} \sqrt{{\ln(\di T/\delta) \ln(T)}{T}} \right).\]
\end{theorem}

\paragraph{Relation to general contextual bandits.} 
There have been recent papers \cite{ADL16, Badanidiyuru14Resourceful}  that solve problems similar to \linCBwK but for general contextual bandits. 
Here the relation between contexts and outcome vectors is arbitrary and the algorithms compete with an arbitrary fixed set of context dependent policies $\Pi$ accessible via an optimization oracle, with regret bounds being $O\left( (\tfrac \OPT B + 1)\sqrt{KT\log(dT|\Pi|/\delta)}\right).$ These approaches could potentially be applied to the linear setting using a set $\Pi$ of linear context dependent policies. 
%\footnote{In fact, there is a lower bound of $\Omega(\sqrt{KT\ln(|\Pi|)})$ on the best possible regret for those problems.}
Comparing their bounds with ours, in our results, essentially a $\sqrt{K\log(|\Pi|)}$ factor is replaced by a factor of $m.$  
Most importantly, we have no dependence on $K$,\footnote{Similar to the regret bounds for  linear contextual bandits \cite{oful, Auer2003, chu2011}.} which enables us to consider problems with large action spaces. In any case, both 
$K$ and $\log(|\Pi|)$ are at least $m$, so their bounds are no smaller.   

Further, suppose that we want to use their result with the set of linear policies, i.e., policies of the form 
\[ \arg\max_{a \in [K]} \{ \cxL_t(a) ^\top \thetaV \},\]
for some fixed $\thetaV \in \Re^m$. Then, their algorithms would require access to an 
``Arg-Max Oracle'' that can find the best such policy  (maximizing total reward) for a given set of contexts and rewards (no resource consumption). 
We show that infact the optimization problem underlying such an ``Arg-Max Oracle" problem is NP-Hard, making such an approach computationally expensive. (Proof is in Appendix \ref{app:hardness}.) 

The only downside to our results is that we need the budget $B$ to be $\Omega(mT^{3/4})$. Getting similar bounds for budgets as small as $B= \Theta(m\sqrt{T})$ is an interesting open problem. (This also indicates that this is indeed a harder problem than all the special cases.)

%\begin{remark} (
\noindent {\bf Near-optimality of regret bounds.} 
{ In \cite{DaniHK08}, it was shown that for the linear contextual bandits problem, no online algorithm can achieve a regret bound better than $\Omega(m\sqrt{T})$. In fact, they prove this lower bound for linear contextual bandits with {\em static} contexts. Since that problem is a special case of the \linCBwK~problem with $d=1$, this shows that the dependence on $m$ and $T$ in the above regret bound is optimal upto log factors. 
For general contextual bandits with resource constraints, the bounds of \cite{ADL16, Badanidiyuru14Resourceful} are near optimal. }
%\end{remark}

%Note that our regret bounds do not depend on $K$ at all, which we are able to avoid by exploiting the linear dependence on the context.

%\subsection{Organization}
%In Section \ref{sec:feasibility}, we illustrate our algorithm and proof techniques for the special case of ``Feasibility problem". In Section \ref{sec:general} we extend it to the general \linCBwCR~problem. And, finally in Section \ref{sec:packing}, we apply the algorithm to the widely studied case of budget constraints (\linCbudget). Due to space constraints, we have eliminated many proofs from the main text. All the missing proofs are appendix.

%This problem turns out to be a common generalization of classic linear contextual bandits problem (linContextual) \cite{Auer2003, chu2011, oful},  bandits with concave rewards and convex knapsacks (BwCR) \cite{AD14}, and the online stochastic convex programming (OSCP) problem \cite{AD15} (more on this in related work section).  We propose an algorithm that achieves provably near-optimal regret for this problem. 

%--------------------------------------------------------------------------------
%\subsection{Challenges and related work \ShipraNIPS{TODO}}
%\label{sec:challenges}
\paragraph{Relation to BwK \cite{AD14} and OSPP \cite{AD15}.}
It is easy to see that the \linCBwK~problem is a generalization of the linear contextual bandits problem \cite{oful, Auer2003, chu2011}. There, the outcome is scalar and the goal is to simply maximize the sum of these. Remarkably, the \linCBwK~problem also turns out to be a common generalization of bandits with knapsacks (BwK) problem considered in \cite{BwK,AD14}, and the online stochastic packing problem (OSPP) studied by \cite{DH09,AWY2009,Feldman10,Devanur2011,AD15}.
%is a common generalization of both the \text{BwCR} problem of \cite{AD14} and the Online Stochastic Convex Programming (OSCP) problem considered in \cite{AD15}. 
In both BwK and OSPP, the outcome of every round $t$ is a reward $r_t$ and a vector $\cv_t$ 
and the goal of the algorithm is to 
maximize $ \sum_{t=1}^T r_t$ while ensuring that $ \sum_{t=1}^T \cv_t \leq B\ones$.
The problems differ in how these rewards and vectors are picked. 
In the OSPP problem, in every round $t$, the algorithm may pick any reward,vector pair 
from a given set $A_t$ of $d+1$-dimensional vectors. 
The set $A_t$ is drawn i.i.d. from an unknown distribution over {\em sets of vectors}.
This corresponds to the special case of \linCBwK, where $m=d+1$ and  the context $\cxL_t(a)$ itself is equal to $(r_t(a),\cv_t(a)$.  
In the BwK problem, there is a fixed set of arms, and for each arm there is an unknown distribution over reward,vector pairs. The algorithm picks an arm and a reward,vector pair is drawn from the corresponding distribution for that arm. 
This corresponds to the special case of \linCBwK, where $m=K$ and 
the context  $X_t = I,$ the identity matrix, for all $t$.

We use techniques from all three special cases: our algorithms follow the primal-dual paradigm using an online learning algorithm to search the dual space, that was established in \cite{AD14}. 
In order to deal with  linear contexts, we use techniques from \cite{oful, Auer2003, chu2011} to estimate the weight matrix $\Meanstar$, and define 
``optimistic estimates'' of $\Meanstar$. 
We also use the technique of combining the objective and the constraints using a certain tradeoff parameter and that was introduced in \cite{AD15}. 
Further new difficulties arise, such as in estimating the optimum value from the first few rounds, a task that follows from standard techniques in 
each of the special cases but is very challenging here. 
We develop a new way of exploration that uses the linear structure, so that one can evaluate all possible choices that could have led to an optimum solution on the historic sample. This technique might be of independent interest in estimating optimum values.  
%\nikhil{More here after we figure it out.}
One can see that the problem is indeed more than the sum of its parts, from the fact that we get the optimal bound for \linCbudget\ only when $B \geq \tilde\Omega(mT^{3/4})$, unlike either special case for which the optimal bound holds for all $B$ (but is meaningful only for $B =\tilde\Omega(m\sqrt{T})$). 

The approach in \cite{AD14} (for \text{BwK}) extends to the case of ``static" contexts,\footnote{It was incorrectly claimed in \cite{AD14} that the approach can be extended to dynamic contexts without much modifications.} 
where each arm has a context that doesn't change over time. 
%On the other hand in the \linCBwK~problem described here, the context $\cxL_t(a)$ of arm $a$ is dynamic and is generated randomly by an unknown distribution at time $t$. 
%An observation that provides some insight into the complexity introduced by dynamic contexts (as compared to static ones) is that 
The OSPP  of \cite{AD15} is \emph{not} a special case of \linCBwK~with static contexts; this is one indication of the additional difficulty of dynamic over static contexts.

\paragraph{Other related work.}
Budget constraints in a bandit setting has recieved considerable attention, but most of the early work focussed on special cases 
such as a single budget constraint in the regular (non-contextual) setting \cite{Ding13Multi,GM2007,Gyorgy07Continuous,MLG2004,LongCCRJ10,TCRJ2012}. 
Recently, \cite{WuSLJ15} showed an $O(\sqrt{T})$ regret in the linear contextual setting with a single budget constraint, when 
costs depend only on contexts and not arms. 
Budget constraints that arise in particular applications such as online advertising \cite{VeeEC12b,PO2006}, 
dynamic pricing \cite{BabaioffDKS15,BesbesZ2009} and  crowdsourcing  \cite{BKS2012,SK2013,SV2013} have also been considered. 
There has also been a long line of work studying special cases of the OSCP problem \cite{DH09,Devanur2011,Feldman10,AWY2009,KTRV14,GuptaM14,VVS10,MY11,KMT11,ChenWang2013}.

Due to space constraints, we have eliminated many proofs from the main text. All the missing proofs are in the appendix.
\section{Preliminaries}
%In this section, we present some useful tools and techniques from prior work.

\subsection{Confidence Ellipsoid}
\label{sec:UCBellipsoid}
Consider a stochastic process which in each round $t$, generates a pair of observations $(r_t, \y_t),$ such that $r_t$ is an unknown linear function of $\y_t$ plus some $0$-mean bounded noise, i.e., $r_t=\MeanVector^\top_* \y_t + \eta_t$, where $\y_t, \MeanVector_* \in \mathbb{R}^m$, $|\eta_t|\le 2R,$ and $ \Ex[\eta_t | \y_1, r_1, \ldots, \y_{t-1}, r_{t-1}, \y_t] = 0. $

At any time $t$, a high confidence estimate of the unknown vector $\MeanVector_*$ can be obtained by building a ``Confidence Ellipsoid" around the $\ell_2$-regularized least-square estimate $\hat{\MeanVector}_t$ constructed from the observations made so far. 
This technique is common in prior work on linear contextual bandits (e.g., in \cite{Auer2003, chu2011, oful}). 
%$$ \forall \beta \in \mathbb{R} \Ex[e^{\beta \eta_t} | \cxL_1, r_1, \ldots, \cxL_{t-1}, r_{t-1}, \cxL_t] \le \exp\left(\frac{\beta^2R^2}{2}\right).$$
%The $R$-sub-Gaussian noise condition automatically implies that $\Ex[\eta_t| \cxL_1, r_1, \ldots, \cxL_{t-1}, r_{t-1}, \cxL_t]=0$. This condition is satisfied by zero-mean noise lying in an interval of length at most $2R$. We will mainly consider zero-mean noise lying in the interval $[0,1]$, but it will also be useful to note that this condition is satisfied by zero-mean noise with variance $2R^2$.
%\shipra{To see this, note that for any $0\le x\le 1$, $e^x \le 1+x+x^2$, and for $x \ge 1$, $e^x \le 1+x+\sum_{n=2}^{\infty} x^n/n! \le 1+x+\sum_{n=2}^{\infty} ({x^2})^{n-1}/(n-1)! = 1+x+\sum_{n=1}^{\infty} (x^2)^n/n!$. }
For any regularization parameter $\lambda>0$, let 
\[ \textstyle M_t :=\lambda I + \sum_{i=1}^{t-1} \y_i \y_i^\top, \text{ and  }\hat{\MeanVector}_t := M_t^{-1} \sum_{i=1}^{t-1} \y_i r_i. \]
%The results we state here are from \cite{oful}.   
The following result from \cite{oful} shows that $\MeanVector_*$ lies with high probability in an ellipsoid with center $\hat{\MeanVector}_t$.
For any positive semi-definite (PSD) matrix $M,$
define the $M$-norm as $\|\MeanVector\|_{M} := \sqrt{\MeanVector^\top M \MeanVector}$. The confidence ellipsoid at time $t$ is defined as
\[ \textstyle C_t := \left\{\MeanVector \in \mathbb{R}^m : \|\MeanVector- \hat{\MeanVector}_t\|_{M_t} \le R\sqrt{m\ln\left( \nicefrac{(1+tm/\lambda)}{\delta}\right)} + \sqrt{\lambda m}\right\}.\]
\begin{lemma}[Theorem 2 of \cite{oful}]
\label{lem:linContextual1}
If $\forall~t$,  $\|\MeanVector_*\|_2 \le \sqrt{m}$ and $\|\y_t\|_2\le \sqrt{m}$, then 
with prob. $1-\delta$, $\MeanVector_*\in C_t.$
 %$b$ is an upper bound on $\y_t$ for all $t$, 
\end{lemma}

Another useful observation about this construction is stated below. It first appeared as Lemma 11 of \cite{Auer2003}, and was also proved as Lemma 3 in \cite{chu2011}. 
\begin{lemma}[Lemma 11 of \cite{Auer2003}]
\label{lem:linContextual2} 
%Let $M_t, t=1, \ldots, T$ is as defined in Lemma \ref{lem:linContextual1}. Then, 
$\sum_{t=1}^T  \|\y_t\|_{ M_t^{-1}}  \le \sqrt{mT\ln(T)}$.
\end{lemma}
As a corollary of the above two lemmas, we  obtain a bound on the total error in the estimate provided by ``any point" from the confidence ellipsoid. 
(Proof is in Appendix \ref{sec:confidenceellipsoidappendix}.)
\begin{corollary}
\label{lem:linContextual3}
For $t=1,\ldots, T$, let $\tilde{\MeanVector}_t\in C_t$ be a point in the confidence ellipsoid, with $\lambda=1, 2R=1$. Then, with probability $1-\delta$,  
\[\textstyle \sum_{t=1}^T |\tilde{\MeanVector}_t^\top \y_t - \MeanVector_*^\top \y_t| \le 2 m \sqrt{T \ln\left( \nicefrac{(1+Tm)}{\delta}\right) \ln(T)} .\]
%\nikhil{Did you mean to have $T$ instead of $t$ inside the $\ln$? Do these earlier papers also have an $m$ instead of $\sqrt{m}$?}
%\shipra{Yes. Yes. One can either have $\sqrt{m \ln(K)}$ or $m$}
\end{corollary}

\subsection{Online Learning}

The online convex optimization (\OCO) problem
considers a $T$ round game played between a learner and an adversary, 
where in round $t$, the learner chooses a $\thetaV_t \in \domainTheta$, and then the adversary picks a concave function $g_t(\thetaV_t): \domainTheta \rightarrow \mathbb{R}$. The learner's choice $\thetaV_t$ may only depend on learner's and adversary's choices in previous rounds. The goal of the learner is to minimize \emph{regret} defined as the difference between the learner's objective value and the value of the best single choice on hindsight:
\[\textstyle \regOCO(T):= \max_{\thetaV \in \domainTheta}\sum_{t=1}^T g_t(\thetaV) -\sum_{t=1}^T g_t(\thetaV_t).\]
%We will use online convex optimization algorithms for solving problems with player's domain of form $||\thetaV_t||_*\le L$, where $||\cdot||_*$ is the dual norm of $||\cdot||$, and $||\cdot||$ will be either the norm used in the distance function $d(\x,S)$, or the norm with respect to which $f$ is $L$-Lipschitz. 
In particular, we will use \emph{linear} reward functions with values in $[-1,1]$, and domain $\Omega$ is the unit simplex in $d+1$ dimensions. 
The algorithm online mirror descent (\OMD) has very fast per step update rules, and provides the following regret guarantees for this setting.
%\nikhil{Are there?}
\begin{lemma}{\cite{Shalev-Shwartz12}}
\label{lem:regOCO}The online mirror-descent algorithm for the \OCO~problem  achieves regret
 $$\regOCO(T) = O(\sqrt{\log(d)T}).$$
\end{lemma}
We actually need the domain to be 
\[ \Omega = \left\{ \thetaV: \|\thetaV\|_1 \leq 1, \thetaV \geq 0 \right\} .\]
This is a special case of a unit simplex in $d+1$ dimensions, by letting the rewards on one of the dimensions always be zero. For the rest of the paper, we assume that the \OMD algorithm is using this domain. 

\comment{
\noindent For optimization over a simplex, the multiplicative weight update algorithm provides stronger guarantees. The step $t$ update of this algorithm is  very fast and efficient, and takes the following form:

\nikhil{Do we need this?}
\begin{equation}
\label{eq:MWupdate}
\thetaV_{t+1,j} = \frac{w_{t,j}}{\sum_j w_{t,j}}, \text{ where } w_{t,j} = w_{t-1,j}(1+\epsilon)^{g_t({\bf e}_j)/M}.
\end{equation}
\begin{lemma}{\cite{AHK12}}
\label{lem:regMW}
For domain $\domainTheta=\{||\thetaV||_1 = 1, \thetaV\ge 0\}$, and given $0\le g_t(\thetaV_t)\le M$, using the multiplicative weight update algorithm we obtain that
for any $\thetaV\in W$,
$$\sum_{t=1}^T g_t(\thetaV_t) \ge (1-\epsilon) \left( \sum_{t=1}^T g_t(\thetaV)\right) -\frac{M \ln(d+1)}{\epsilon},$$
\end{lemma}
}
%-------------------------------------------------------------------------------------------------

\section{Algorithm}

\subsection{Optimistic estimates of unknown parameters}
\label{sec:UCB}
\newcommand{\Ell}{{\cal G}}
\newcommand{\meanstarj}{\mathbf{w}_{*j}}
\newcommand{\meanhattj}{{\mathbf{\hat{w}}_{tj}}}
\newcommand{\mean}{\mathbf{w}}

%We employ the techniques from Section \ref{sec:UCBellipsoid} to construct a  confidence ellipsoid for every column of $\Meanstar$. 
 Let $a_t$ denote the arm played by the algorithm at time $t$. In the beginning of every round, we use the outcomes and contexts from previous rounds to construct a 
confidence ellipsoid for $\MeanVector_*$ and every column of $\Meanstar$. The construction of confidence ellipsoid for $\MeanVector_*$ follows directly from the techniques in \prettyref{sec:UCBellipsoid} with $y_t=\cxL_t(a_t)$ and $r_t$ being reward at time $t$. To construct a confidence ellipsoid for a column $j$ of $\Meanstar$, we use the techniques in \prettyref{sec:UCBellipsoid} while substituting $\y_t=\cxL_t(a_t)$ and $r_t=\cv_t(a_t)_j$ for every $j$. 

As in \prettyref{sec:UCBellipsoid}, let $M_t :=I + \sum_{i=1}^{t-1} \cxL_i(a_i) \cxL_i(a_i)^\top$, and construct the regularized least squares estimate for  $\rmu_*, \Meanstar$, respectively, as
\begin{eqnarray}
\label{eq:EstMean}
\textstyle \hat{\MeanVector}_t & := & \textstyle M_{t}^{-1} \sum_{i=1}^{t-1} \cxL_{i}(a_i) r_{i}(a_i)^\top\\
\textstyle \MeanEst_t & := & \textstyle M_{t}^{-1} \sum_{i=1}^{t-1} \cxL_{i}(a_i) \cv_{i}(a_i)^\top.
\end{eqnarray}
Define confidence ellipsoid for parameter $\MeanVector_*$ as
\[C_{t,0} := \left\{\MeanVector \in \mathbb{R}^m : \|\MeanVector- \hat{\MeanVector}\|_{M_t} \le \sqrt{m\ln\left( \nicefrac{(\di+tm\di)}{\delta}\right)} + \sqrt{m}\right\},\]
and optimistic estimate of $\MeanVector_*$ for every arm $a$ as:
\begin{equation}
\label{eq:UCBreward}
\textstyle
\tilde{\MeanVector}_{t}(a) := \arg \max_{\MeanVector \in C_{t,0}} \cxL_t(a)^\top\MeanVector.
\end{equation}
Let $\mean_j$ denote the $j^{th}$ column of a matrix $\Mean$. We define a confidence ellipsoid for each column $j$, as 
\[C_{t,j} := \left\{\mean \in \mathbb{R}^m : \|\mean- \meanhattj\|_{M_t} \le \sqrt{m\ln\left( \nicefrac{(\di+tm\di)}{\delta}\right)} + \sqrt{m}\right\},\]
and denote by $\Ell_t$, the Cartesian product of all these ellipsoids:
$\Ell_t:=\{\Mean \in \mathbb{R}^{m\times\di}: \mean_j \in C_{t,j}\}.$ 
Note that Lemma \ref{lem:linContextual1} implies $\Mean_*\in \Ell_t$ with probability $1-\delta$. 
Now, given a vector $\thetaV_t \in \mathbb{R}^\di$, we define the {\em optimistic estimate} of weight matrix at time $t$ {\it w.r.t.} $\thetaV_t$, for every arm $a\in [K]$, as :
\begin{equation}
\label{eq:UCBEstMean}
\textstyle
\MeanUCB_{t}(a) := \arg \min_{\Mean \in \Ell_t} \cxL_t(a)^\top\Mean \thetaV_t.
\end{equation}
Intuitively, for reward we want an upper confidence bound and for consumption we want a lower confidence bound as an optimistic estimate. This intuition aligns with the above definitions, where the maximizer was used in case of reward and a minimizer was used for consumption. The utility and precise meaning of $\thetaV_t$ will become clearer when we describe the algorithm and present regret analysis. 

%We describe the choice of vectors $\thetaV_1, \thetaV_2, \ldots,\thetaV_T$ later. 
Using the definition of $\tilde{\MeanVector}_t, \MeanUCB_t$, along with the results in Lemma \ref{lem:linContextual1} and Corollary \ref{lem:linContextual3} about confidence ellipsoids, the following can be derived. 
%demonstrating that $\MeanUCB_t(a)$ indeed provides an optimistic estimate of $\Mean^*$.
\begin{corollary}\label{cor:optest}
	With probability $1-\delta$, for any sequence of $\thetaV_1, \thetaV_2, \ldots, \thetaV_T$, 
	\begin{enumerate}
		\item $\cxL_t(a)^\top\tilde{\MeanVector}_t(a) \ge \cxL_t(a)^\top\tilde{\MeanVector}$, for all arms $a\in[K]$, for all time $t$.
		\item $\cxL_t(a)^\top\MeanUCB_t(a) \thetaV_t \le \cxL_t(a)^\top\Meanstar \thetaV_t$, for all arms $a\in[K]$, for all time $t$. 
		\item $|\sum_{t=1}^T (\tilde \MeanVector_t(a_t) - \MeanVector_*)^\top \cxL_t(a_t)| \le \left(2m\sqrt{T \ln\left( \nicefrac{(1+tm)}{\delta}\right) \ln(T)}\right).$
			\item $\|\sum_{t=1}^T (\MeanUCB_t(a_t) - \Meanstar)^\top \cxL_t(a_t)\| \le \oneNorm \left(2m\sqrt{T \ln\left( \nicefrac{(d+tmd)}{\delta}\right) \ln(T)}\right).$
	\end{enumerate}
\end{corollary}

Essentially, the first two claims ensure that we have optimistic estimates, and the last two claims ensure that the estimates quickly converge to the true parameters.

\subsection{The core algorithm}
\label{sec:coreAlgo}
{In this section, we present an algorithm, and analysis, under the assumption that a certain parameter $Z$ is given. 
Later, we show how to use the first $T_0$ rounds to estimate $Z$, 
and also bound the additional regret due to these $T_0$ rounds.
We define $Z$ now. }
\begin{assumption}\label{assum:Z}
	Assume we are given $Z$ such that 
	$\tfrac\OPT B \leq Z  \leq O(\tfrac \OPT B + 1)$.
	%\leq O(1) \tfrac{\OPT}{B} .         \]
\end{assumption}

%\shipra{Later, we discuss how to estimate $Z$ to desired accuracy for some special cases. }
The algorithm constructs estimates $\hat{\MeanVector}_t$ and 
$\MeanEst_t $  as in \prettyref{sec:UCB}.
It also runs the OMD algorithm for an instance of the online learning problem, over the unit simplex.
%\footnote{This is simply the well known multiplicative weight update algorithm.} 
The vector played by the online learning algorithm in time step $t$ is $\thetaV_t$. 
After observing the context, the optimistic estimates for each arm are then constructed using $\thetaV_t$, as defined in (\ref{eq:UCBreward}) and (\ref{eq:UCBEstMean}). Intuitively, $\thetaV_t$ is used here as a multiplier to combine different columns of the weight matrix, to get an optimistic weight vector for every arm.
%In the estimates  for mean consumption vector (refer to \eqref{eq:UCBEstMean}), $\thetaV_t$ is used to combine different components of the consumption vector estimates for every arm $a$, to obtain a single number $\MeanUCB_t(a)$ representing the weighted consumption estimate for that arm.  
An {\it adjusted estimated reward} for arm $a$ is then defined by using $Z$ to combine optimistic estimate of reward with optimistic estimate of consumption, as
 $ (\cxL_t(a)^\top \tilde{\MeanVector}_{t}(a)) - Z (\cxL_t(a)^\top \MeanUCB_t(a)\thetaV_t) .$
The algorithm chooses the arm which appears to be the best according to adjusted estimated reward.
After observing the resulting reward and consumption vectors, the estimates are updated. The online learning algorithm is advanced by one step, by defining the profit vector to be $\cv_t(a_t) - \tfrac B T \ones$.  
 The algorithm ends either after $T$ time steps or as soon as the total consumption exceeds the budget along some dimension. 
 
 %There are two main differences: first, online learning updates are needed to predict Fenchel dual variable for the objective in addition to predicting those for constraints as in the previous section. We use that for $L$-Lipschitz function $f$ Fenchel duality (refer to Section \ref{sec:Fenchel}) implies $f(\cv) = \min_{\|\phiV\|_*\le L} f^*(\phiV) - \phiV^\top\cv$. Second, we need to use the parameter $Z$ to combine the objective with constraints. 
%----------------------------------------------------------------------

\begin{algorithm}[t] 
\caption{Algorithm for \linCBwK, with given $Z$}
\label{algo:coreAlgo}
  \begin{algorithmic}
	\STATE
	\STATE Initialize $\thetaV_{1}$ as per the \OCO algorithm. %$\epsilon = \sqrt{\frac{\log(d+1)}{B}}$. 
	\STATE Initialize $Z$ such that $\tfrac{\OPT}{B} \leq Z \leq O(\tfrac \OPT B + 1)$. 
	\FORALL{$t=1,..., T$} 
	%\STATE Set $\tilde{W}_{ij}(t)=\UCB(W_{ij})$ for $\theta_j<0$, and $\tilde{W}_{ij}(t)=\LCB(W_{ij})$ for $\theta_j\ge 0$.
	\STATE Observe $\cx_t$. %for all $a\in A$.
	\STATE For every $a\in [K]$, compute $\tilde{\MeanVector}_{t}(a) $ and $\MeanUCB_{t}(a)$ as per (\ref{eq:UCBreward}) and (\ref{eq:UCBEstMean}) respectively. %Let $\vt(a) := \MeanUCB_t(a)^\top \cxL_t(a)$. 
	\STATE Play the arm 
	$a_t :=\arg \max_{a \in [K]}   \cxL_t(a)^\top (\tilde{\MeanVector}_{t}(a) - Z\MeanUCB_t(a)\thetaV_t) .$	
		\STATE Observe $r_t(a_t)$ and $\cv_t(a_t).$
		\STATE If for some $j = 1..d, \sum_{t'\leq t} \cv_{t'}(a_{t'})\cdot {\bf e}_j \geq B $ then EXIT.  
			\STATE Use $\cxL_{t}(a_t), r_t(a_t)$ and $\cv_t(a_t)$ to obtain $\hat{\MeanVector}_{t+1}, \MeanEst_{t+1}$ and $\Ell_{t+1}$. 
			\STATE Update  $\thetaV_{t+1}$ as per the \OCO algorithm with $g_t(\thetaV_t) := \thetaV_t \cdot \left(\cv_t(a_t) -\frac{B}{T}{\bf 1}\right).$
			\ENDFOR
			\STATE

%  	}
  	\end{algorithmic}
\end{algorithm}
%----------------------------------------------------
%\nikhil{I changed the algo above. Please check everything.}

\begin{theorem}
\label{th:coreAlgo}
Given a $Z$ as per  Assumption \ref{assum:Z}, Algorithm \ref{algo:coreAlgo} achieves the following bounds, given that $\regOCO(T)$ is the regret of the \OCO algorithm,  with probability $1-\delta$:
$$ \regret(T) \le  O\left( (\tfrac \OPT B + 1) m \sqrt{T \ln(\di T/\delta) \ln(T)} \right).$$ 
\end{theorem}
\noindent{\em (Proof Sketch)}
	We provide a sketch of the proof here, with the full proof in  Appendix \ref{app:coreAlgo}. 
	Let $\tau$ be the stopping time of the algorithm.  
	The proof is in 3 steps: 
	\paragraph{Step 1:}
	Since $\Ex[\cv_t(a_t) | \cx_t, a_t, H_{t-1}] = \Meanstar^\top \cxL_t(a_t)$, we apply Azuma-Hoeffding to get that with high probability 
	$	\left\|\sum_{t=1}^\tau \cv_t(a_t)-\Meanstar^\top \cxL_t(a_t)\right\|_\infty$ is small. 
	Similarly, a lower bound on  the sum of $\mu_*^\top \cxL_t(a_t)$ is sufficient. 
	
	\paragraph{Step 2:} From Corollary \ref{cor:optest}, with high  probability, we can bound 
	$\left\|  \sum_{t=1}^T  (\Meanstar - \MeanUCB_t(a_t))^\top\cxL_t(a_t) \right \|_\infty
	.$
	%\textcolor{red}{Nikhil: We seem to get an $m$ here instead of $\sqrt{m}$.}
	It is therefore sufficient to work with  the sum of the vectors $\MeanUCB_t(a_t)^\top\cxL_t(a_t)$, 
	and similarly  $\tilde{\MeanVector}_t(a_t)^\top\cxL_t(a_t)$. 

	\paragraph{Step 3:} 
	The proof is completed by showing the desired bound on  
	$\OPT- \sum_{t=1}^\tau \tilde{\MeanVector}_t(a_t)^\top\cxL_t(a_t).$ This part is similar to the online stochastic packing problem; if the actual reward and consumption vectors were $\tilde{\MeanVector}_t(a_t)^\top\cxL_t(a_t)$  and $\MeanUCB_t(a_t)^\top\cxL_t(a_t)$, then it would be exactly like that problem. We adapt techniques from \cite{AD15}: use the \OCO algorithm and the $Z$ parameter to combine constraints into the objective. If a dimension is being consumed too fast, then the multiplier for that dimension should increase, making the algorithm to pick arms that are not likely to consume too much along this dimension.

\subsection{Algorithm with $Z$ computation}
\label{sec:packing}

In this section, we present a modification of Algorithm 1 which computes the required parameter $Z$ and therefore does not need to be provided with a $Z$ as input, as assumed previously in Assumption \ref{assum:Z}. The algorithm computes $Z$ using the observations from first $T_0$ rounds. Once $Z$ is computed, the algorithm from the previous section can be run for the remaining time steps. However, it needs to be modified slightly to take into account the budget consumed during the first $T_0$ rounds. We handle this by using a smaller budget $B' = B-T_0$ in the computations for remaining rounds. 
%The algorithm stops if the full actual budget $B$ is consumed for any dimension (including consumption in the first $T_0$ rounds). 
The modified algorithm is given below. %The choice of $T_0$, and workings of the first $T_0$ rounds will be explained subsequently. 

\begin{algorithm}[H] 
\caption{Algorithm for \linCBwK, with $Z$ computation}
\label{algo:linearZ}
  \begin{algorithmic}
	\STATE {\bf Inputs:} $B, T_0, B'=B-T_0$
	\STATE Using observations from first $T_0$ rounds, compute $Z$ such that $\tfrac{\OPT}{B'} \leq Z\le O(\tfrac{\OPT}{B'} +1)$.
	\STATE Run Algorithm \ref{algo:coreAlgo} for $T-T_0$ rounds and budget $B'$.
	\comment{
	\STATE Initialize $\thetaV_{1}= \tfrac 1 {d+1} \ones$. %$\epsilon = \sqrt{\frac{\log(d+1)}{B}}$.
	\FORALL{$t=T_0+1,..., T$} 
	%\STATE Set $\tilde{W}_{ij}(t)=\UCB(W_{ij})$ for $\theta_j<0$, and $\tilde{W}_{ij}(t)=\LCB(W_{ij})$ for $\theta_j\ge 0$.
	\STATE Observe $\cx_t$. %for all $a\in A$.
	\STATE For every $a\in [K]$, compute $\tilde{\MeanVector}_{t}(a) $ and $\MeanUCB_{t}(a)$.
	%$$\tilde{\MeanVector}_{t}(a) = \arg \max_{\MeanVector \in C_{t,0}} \cxL_t(a)^\top\MeanVector$$ and $$\MeanUCB_{t}(a) := \arg \min_{\Mean \in \Ell_t} \cxL_t(a)^\top\Mean\thetaV_t.$$  %Let $\vt(a) := \MeanUCB_t(a)^\top \cxL_t(a)$. 
	\STATE Play the arm 
	$a_t :=\arg \max_{a \in [K]}   \cxL_t(a)^\top (\tilde{\MeanVector}_{t}(a) - Z\MeanUCB_t(a)\thetaV_t) .$	
		\STATE Observe $r_t(a_t)$ and $\cv_t(a_t).$
		\STATE If for some $j = 1..d, \ \sum_{t=1}^T \cv_{t'}(a_{t'})\cdot {\bf e}_j \geq B $ then EXIT.  
			\STATE Use $\cxL_{t}(a_t), r_t(a_t)$ and $\cv_t(a_t)$ to obtain $\hat{\MeanVector}_{t+1}, \MeanEst_{t+1}$ and $\Ell_{t+1}$.
			\STATE Update  $\thetaV_{t+1}$ as per the \OCO algorithm with $g_t(\thetaV_t) := \thetaV_t \cdot \left(\vt -\frac{B'}{T}{\bf 1}\right).$
			%\STATE Update  $\thetaV_{t+1}$ using multiplicative weight update: 
	%\begin{center}$\forall\ j = 1..d, w_{t,j} = w_{t-1,j}(1+\epsilon)^{\cv_t(a_t)\cdot e_j - B'/T}  $\end{center}
	%and 
	%\begin{center}$\forall\ j = 1..d, \thetaV_{t+1,j} = \frac{w_{t,j}}{1+\sum_{j'=1}^d w_{t,j'}},$\end{center}
			\ENDFOR
			\STATE

		}
%  	}
  	\end{algorithmic}
\end{algorithm}

Next, we provide details of the first $T_0$ rounds, and choice of $T_0$.

We provide a method that takes advantage of the linear structure of the problem, and explores in the $m$-dimensional space of contexts and weight vectors to obtain bounds independent of $K$. We use the following procedure.  In every round $t=1,\ldots, T_0$, after observing $\cx_t$, let  $p_t \in \Delta^{[K]}$ be 
\begin{eqnarray}
 p_t & := & \arg\max_{p  \in \Delta^{[K]}} \|\cx_t p\|_{M_t^{-1}} \label{eq:exploreChoice},\\
\text{where } M_t & := & \textstyle   I + \sum_{i=1}^{t-1} (\cx_i p_i) (\cx_i p_i)^\top.
\end{eqnarray}
Select arm $a_t=a$ with probability $p_t(a)$. In fact, since $M_t$ is a PSD matrix, due to convexity of the function $\|\cx_t p\|_{M_t^{-1}}^2$, it is the same as playing $a_t = \arg\max_{a\in [K]} \|\cxL_t (a)\|_{M_t^{-1}}$. 
Construct estimates $\hat{\MeanVector}, \MeanEst_t$ of $\MeanVector_*, \Mean_*$ at time $t$ as 
\[\textstyle \hat{\MeanVector}_t:= M_t^{-1} \sum_{i=1}^{t-1} (\cx_i p_i) r_{i}(a_i), \ \ \MeanEst_{t} := M_t^{-1} \sum_{i=1}^{t-1} (\cx_i p_i) \cv_{i}(a_i)^\top.\]
And, for some value of $\gamma$ defined later, obtain an estimate $\OPTest^\gamma$ of $\OPT$ as:
\begin{eqnarray}
\label{eq:OPTestProgram}
\textstyle
\OPTest^{\gamma} & := & \begin{array}{rcl}
 \max_{\pi} & \frac{T}{T_0} \sum_{i=1}^{T_0} \hat{\MeanVector}_i^\top \cx_i \pi(\cx_i)& \\
\text{such that} &  \frac{T}{T_0}\sum_{i=1}^{T_0}\MeanEst_i^\top \cx_i \pi(\cx_i) \le B+\gamma.& 
\end{array}
\end{eqnarray}

For an intuition about the choice of arm in \eqref{eq:exploreChoice}, observe from the discussion in Section \ref{sec:UCBellipsoid} that every column $\meanstarj$ of $\Mean_*$ is guaranteed to lie inside the confidence ellipsoid centered at column $\meanhattj$ of $\MeanEst_t$, namely the ellipsoid, $\|\mean-\meanhattj\|_{M_t}^2 \leq 4m\ln(Tm/\delta)$. Note that this ellipsoid has principle axes as eigenvectors of $M_t$, and the length of semi-principle axes is given by {\em inverse} eigenvalues of $M_t$. Therefore, by maximizing $\|\cx_t p\|_{M_t^{-1}}$ we are choosing the context closest to the direction of the longest principal axes of the confidence ellipsoid, i.e. in the direction of maximum uncertainty. Intuitively, this corresponds to pure exploration: by making an observation in the direction where uncertainty is large we can reduce the uncertainty in our estimate most effectively. 

A more algebraic explanation is as follows. For a good estimation of $\OPT$ by $\OPTest^\gamma$, we want the estimates $\MeanEst_t$ and $\Mean_*$ (and, $\hat{\MeanVector}$ and $\MeanVector_*$) to be close enough so that $\|\sum_{t=1}^{T_0}(\MeanEst_t- \MeanEst_*)^\top \cx_t \pi(\cx_t)\|_\infty$ (and, $|\sum_{t=1}^{T_0}(\hat{\MeanVector}_t- \MeanVector_*)^\top \cx_t \pi(\cx_t)|$) is small for all policies $\pi$, and in particular for sample optimal policies. Now, using Cauchy-Schwartz these are bounded by
\[\textstyle \sum_{t=1}^{T_0} \|\hat{\MeanVector}_t-\MeanVector_*\|_{M_t} \| \cx_t \pi(\cx_t))\|_{M_t^{-1}}, \text{ and }\]
 \[\textstyle \sum_{t=1}^{T_0} \|\MeanEst_t-\Mean_*\|_{M_t} \| \cx_t \pi(\cx_t))\|_{M_t^{-1}}, \]
where we define $\|W\|_M$, the $M$-norm of matrix $W$ to be the max of column-wise $M$-norms. Using Lemma \ref{lem:linContextual1}, the term  $\|\hat{\MeanVector}_t-\MeanVector_*\|_{M_t}$  is bounded by $2\sqrt{m\ln(T_0m/\delta)}$ , and $\|\MeanEst_t-\Mean_*\|_{M_t}$ is bounded by $2\sqrt{m\ln(T_0md/\delta)}$, with probability $1-\delta$. 
Lemma \ref{lem:linContextual2} bounds the second term $\sum_{t=1}^{T_0} \|\cx_t \pi(\cx_t)\|_{M_t^{-1}}$ but only when $\pi$ is the played policy. This is where we use that the played policy $p_t$ was chosen to maximize $\| \cx_t p_t\|_{M_t^{-1}}$, so that $\sum_{t=1}^{T_0} \|\cx_t \pi(\cx_t)\|_{M_t^{-1}}\le \sum_{t=1}^{T_0} \| \cx_t p_t\|_{M_t^{-1}}$ and the bound $\sum_{t=1}^{T_0} \| \cx_t p_t\|_{M_t^{-1}} \le \sqrt{mT_0 \ln(T_0)}$ given by Lemma \ref{lem:linContextual2} actually bounds $\sum_{t=1}^{T_0} \|\cx_t \pi(\cx_t)\|_{M_t^{-1}}$ for all $\pi$.
Combining, we get a bound of $\gammaValue$ on deviations $\|\sum_{t=1}^{T_0}(\MeanEst_t- \MeanEst_*)^\top \cx_t \pi(\cx_t)\|_\infty$ and $|\sum_{t=1}^{T_0}(\hat{\MeanVector}_t- \MeanVector_*)^\top \cx_t \pi(\cx_t)|$ for all $\pi$.

We prove the following lemma.
\begin{lemma}
\label{lem:estimatingOPT}
For $\gamma=\left(\frac{T}{T_0}\right)\gammaValue$, with probability $1-O(\delta)$,
 \begin{center} $ \OPT - 2\gamma \le \OPTest^{2\gamma} \le \OPT+9\gamma(\frac{\OPT}{B}+1).$ \end{center}
\end{lemma}
%For \linCbudget, $Z^*\le \frac{\OPT}{(B/T)}, L=1, \oneNorm=\oneNorm_{\infty}=1$. 
%Therefore, $Z=\frac{(\OPTest^{2\gamma} + L\gamma)}{(B/T)} +1$ satisfies the bounds of Lemma \ref{lem:packingZest}. 
%The proof of Theorem \ref{th:packing} uses the following estimate for $Z$ obtained  as described in Section \ref{sec:OPTestimation}.

\begin{corollary}
\label{cor:packingZest}
%Let $\gamma=\gammaValueInfty$. 
Set $Z=\frac{(\OPTest^{2\gamma} + 2\gamma)}{B} +1$, with above value of $\gamma$. Then, with probability $1-O(\delta)$,
%Using the first $T_0$ rounds for exploration, we can obtain $Z$ such that with probability $1-O(\delta)$,
%\begin{center} $ \OPT \le \OPTest \le (1+\frac{7\gamma}{(B/T)})\OPT + 7\gamma.$ \end{center}
\begin{center} $ \frac{\OPT}{B}+1 \le Z \le (1+\frac{11\gamma}{B}) (\frac{\OPT}{B}+1).$ \end{center}

\end{corollary}

%\shipra{Above requires proof in appendix..?}
%\nikhil{We should just say once in the beginning that all missing proofs are in the appendix and be done with it.}
Corollary \ref{cor:packingZest} implies that as long as $B\ge \gamma$, i.e., $B \ge \tilde{\Omega}(\frac{mT}{\sqrt{T_0}})$, $Z$ is a constant factor approximation of $ \tfrac{\OPT}{B}+1 \ge Z^*$, therefore Theorem \ref{th:coreAlgo} should provide an $\tilde{O}\left( (\tfrac{\OPT}{B}+1) m\sqrt{T}\right)$ regret bound. However, this bound does not account for the budget consumed in the first $T_0$ rounds. Considering that (at most) $T_0$ amount can be consumed from the budget in the first $T_0$ rounds, we have an additional regret of $\frac{\OPT}{B} T_0$. Further, since we have $B'=B-T_0$ budget for remaining $T-T_0$ rounds, we need a $Z$ that satisfies the required assumption for $B'$ instead of $B$ (i.e., we need $\frac{\OPT}{B'} \le Z\le O(1) \left(\frac{\OPT}{B'}+1\right))$. If $B\ge 2T_0$, then, $B'\ge B/2$, and using $2$ times the $Z$ computed in Corollary \ref{cor:packingZest} would satisfy the required assumption.

Together, these observations give Theorem \ref{th:largeK}.

\begin{theorem}
\label{th:largeK}
 Using Algorithm \ref{algo:linearZ} with $T_0$ such that $B> \max\{2T_0, mT/\sqrt{T_0}\}$, and twice the $Z$ given by Corollary \ref{cor:packingZest}, we get a high probability regret bound of 
	%\[  O\left(\frac{T_0}{B} \OPT + \left(\frac{\OPT}{(B/T)} +1\right)\left(1+\frac{\gamma}{(B/T)}\right)\frac{m\oneNorm}{\sqrt{T_0}}\right) \]
	%where $\gamma=\gammaValue$. 
	\[\textstyle \tilde{O}\left(\left(\frac{\OPT}{B} +1\right) \left({T_0} + {m}{\sqrt{T}}\right)\right).\]
	In particular, using $T_0= \sqrt{T}$, and assuming $B>m T^{3/4}$ gives a regret bound of 
	\[\textstyle \tilde{O}\left(\left(\frac{\OPT}{B} +1\right) {m}{\sqrt{T}}\right).\]
	
	\comment{There is also an alternative version of this result, which only requires $B>B'=T_0$, to get regret bound of 
	\[\displaystyle \tilde{O}\left(\left(\nicefrac{\OPT}{(B/T)} +1\right) \left(\nicefrac{T_0}{T} + m\sqrt{\nicefrac{T}{T_0^3}} + m\sqrt{\nicefrac{1}{T}}\right)\right)\]
	Using $T_0= T^{3/5}$, assuming $B>T^{3/5}$, this gives
	\[\tilde{O}\left(\left(\nicefrac{\OPT}{(B/T)} +1\right) \nicefrac{m}{T^{2/5}}\right)\]
	}
\end{theorem}

\bibliographystyle{abbrv}
\bibliography{bibliography_contextual}

\newpage
\appendix
\begin{center}
{\bf Appendix}
\end{center}

\section{Concentration Inequalities}
\begin{lemma}[Azuma-Hoeffding inequality]
\label{lem:azuma}
If a super-martingale $(Y_t; t\ge 0)$, corresponding to filtration ${\cal F}_t$, satisfies $|Y_t-Y_{t-1}| \le c_t$ for some constant $c_t$, for all $t=1, \ldots, T$, then for any $a\ge 0$,
	$$\Pr(Y_T-Y_0 \ge a) \le e^{-\frac{a^2}{2\sum_{t=1}^T c_t^2}}.$$
\end{lemma}
%======================================================
\section{Benchmark}\label{app:benchmark}
\begin{proof}[Proof of Lemma \ref{lem:staticOPT}]
	%Let $\overline{\OPT}$ denote the value of optimal adaptive solution that knows the distribution ${\cal D}$ and the weight matrix $\Mean_*$, and can take into account the history upto that point as well as the current context to decide (possibly with randomization) which arm to pull at time $t$. The expected distance of the optimal adaptive solution from the constraint set $S$ is assumed to be $0$. We show that there exists a static policy $q^*$ such that $f(\cvA(q^*)) \ge \overline{\OPT}$, and $\cvA(q^*)\in S$.
	
	For an instantiation $\omega=(\cx_t, \cvm_t)_{t=1}^T$ of the sequence of inputs, let vector ${\bf p}^*_t(\omega) \in \Delta^{K+1}$ denote the distribution over actions (plus no-op) taken by the {\it optimal adaptive policy} at time $t$. Then, %by concavity of function $f$,
	\begin{equation}
	\label{eq:b1}
	\textstyle \overline{\OPT} = \Ex_{\omega \sim {\cal D}^T}[\sum_{t=1}^T  \rv_t^\top {\bf p}^*_t(\omega)] 
	\end{equation}
	Also, since this is a feasible policy, 
	\begin{equation}
	\label{eq:b2}
	\Ex_{\omega \sim {\cal D}^T}[\sum_{t=1}^T  V_t^\top {\bf p}^*_t(\omega)] \le B\ones 
	\end{equation}
	Construct a {\it static} context dependent policy $\pi^*$ as follows: for any $\cx\in [0,1]^{m\times K}$, define
	$$\pi^*(\cx) := \frac{1}{T} \sum_{t=1}^T\Ex_\omega[ {\bf p}^*_t(\omega) | \cx_t=\cx].$$
	Intuitively, $\pi^*(\cx)_a$ denotes (in hindsight) the probability that the optimal adaptive policy takes an action $a$ when presented with a context $\cx$, averaged over all time steps.
	Now, by definition of $\rv(\pi),\cvA(\pi)$, from above definition of $\pi^*$, and \eqref{eq:b1}, \eqref{eq:b2},
	$$T \rv(\pi^*) = T \Ex_{\cx \sim {\cal D}}[ \MeanVector_*^\top \cx \pi^*(\cx)] = \textstyle \Ex_\omega[\sum_{t=1}^T \cvm_t {\bf p}^*_t(\omega)] = \overline{\OPT},$$
	$$T \cvA(\pi^*) = T \Ex_{\cx \sim {\cal D}}[ \Meanstar^\top \cx \pi^*(\cx)] = \textstyle \Ex_\omega[\sum_{t=1}^T \cvm_t {\bf p}^*_t(\omega)] \le B\ones,$$

\end{proof}

\section{Hardness of linear AMO} 
\label{app:hardness}
In this section we show that finding the best linear policy is NP-Hard. 
The input to the problem is, for each $t \in [T],$ and each arm $a \in[K]$,  a  context ${\cxL}_t(a) \in \domain^{m}$, and a reward $r_t(a) \in [-1,1]$. The output is a vector $\thetaV \in \Re^m$ that maximizes 
$\sum_t r_t(a_t)$ where 
\[a_t = \arg \max _{a\in [K]}\{ \cxL_t(a)^\top \theta \} .\] 

We give a reduction from the problem of learning halfspaces with noise \cite{guruswami2009hardness}. The input to this problem is
for some integer $n$, for each $i \in [n]$, 
a vector $z_i \in \domain^m$, and $y_i \in \{-1,+1\}$. 
The output is a vector $\thetaV \in \Re^m$ that maximizes 
\[ \sum_{i=1}^{n} sign(\z_i ^\top \thetaV) y_i . \] 

Given an instance of the problem of learning halfspaces with noise, construct an instance of the linear AMO as follows. 
The time horizon $T=n$, and the number of arms $K=2$. 
For each $t\in[T]$, the context of the first arm, $\cxL_t(1) = z_t$, 
and its reward $r_t(1) = y_t$. The context of the second arm, 
$\cxL_t(2) = {\mathbf{0}}$, the all zeroes vector, and the reward $r_t(2)$ is also 0. 

The total reward of a linear policy w.r.t a vector $\thetaV$ for this instance is 
\[ |\{ i: sign(z_i^\top\thetaV) = 1, y_i =1 \}| - |\{ i: sign(z_i^\top\thetaV) = 1, y_i = -1 \}|. \] 
It is easy to see that this is an affine transformation of the objective for the problem of learning halfspaces with noise. 

%======================================================
\section{Confidence ellipsoids}
\label{sec:confidenceellipsoidappendix}
\begin{proof}[\bf Proof of Corollary \ref{lem:linContextual3}]
	The following holds with probability $1-\delta$. 
	\begin{eqnarray}
	%\label{eq:linContextual3}
	\sum_{t=1}^T |\tilde{\MeanVector}_t^\top \cxL_t - \MeanVector_*^\top \cxL_t| & \le & \sum_{t=1}^T \|\tilde{\MeanVector}_t - \MeanVector_*\|_{M_t} \| \cxL_t \|_{M_t^{-1} }\nonumber\\
	%& \le & \left(R\sqrt{m\ln\left( \frac{1+tm/\lambda}{\delta}\right)} + \sqrt{\lambda m} \right) \sum_{t=1}^T  \sqrt{\cxL_t M_t^{-1} \cxL_t} \nonumber\\
	& \le & \left(\sqrt{m\ln\left( \frac{1+tm}{\delta}\right)} + \sqrt{ m} \right)  \sqrt{mT\ln(T)}.\nonumber
	%& \le & \left(\sqrt{m\ln\left( \frac{1+tb^2/\lambda}{\delta}\right)} + \sqrt{\lambda} \|\MeanVector_*\|_2 \right) \sum_{t=1}^T  \sqrt{\cxL_t M_t^{-1} \cxL_t} \nonumber\\
	%& \le & \le \left(\sqrt{m\ln\left( \frac{1+tb^2/\lambda}{\delta}\right)} + \sqrt{\lambda} \|\MeanVector_*\|_2 \right) \left(  \sqrt{mT\ln(T)}\right)\nonumber
	\end{eqnarray}
	
	The inequality in the first line is a matrix-norm version of Cauchy-Schwartz
	(\prettyref{lem:MatrixCS}). 
	%The equality in the first line is due to the fact that $M$ is PSD and the inequality in the second line follows from Cauchy-Schwartz. The inequalities in third and fourth line use the results stated in 
	The inequality in the second line is due to Lemmas \ref{lem:linContextual1} and \ref{lem:linContextual2}. 
	The lemma follows from multiplying out the two factors in the second line.
	%Then, substituting $\lambda=1$, we obtain the desired bound.
	
	\end{proof}
	\newcommand{\av}{\mathbf{a}}
	\newcommand{\bv}{\mathbf{b}}
	\newcommand{\Minv}{{M^{-1}}}
	\newcommand{\Mhalf}{M_{1/2}}
	\newcommand{\Minvhalf}{M_{-1/2}}
	\begin{lemma}\label{lem:MatrixCS}
		For any positive definite matrix $M\in \Real^{n\times n}$ and any two vectors 
		$\av,\bv\in \Real^n$, $|\av^\top \bv| \leq \|\av\|_M \| \bv\|_\Minv$. 
		\end{lemma}
		\begin{proof}
			Since $M$ is positive definite, there exists a matrix $\Mhalf$ such that $M = \Mhalf \Mhalf^\top$. Further, $\Minv = \Minvhalf^\top \Minvhalf$ where $\Minvhalf = \Mhalf^{-1}$. 
			$$\|\av^\top\Mhalf\|^2 = \av^\top\Mhalf \Mhalf^\top \av = \av^\top M \av = \|\av\|_M^2 .$$ 
			Similarly, $\|\Minvhalf \bv\|^2 = \| \bv\|_{\Minv}^2.$ Now applying Cauchy-Schwartz, we get that
			\[ |\av^\top \bv| = |\av^\top \Mhalf \Minvhalf \bv | \leq \|\av^\top\Mhalf \| \|\Minvhalf \bv \| = \|\av\|_M \|\bv\|_\Minv .\]
			
			\end{proof}

\begin{proof}[\bf Proof of Corollary \ref{cor:optest}]
			Here, the first claim follows simply from definition of $\MeanUCB_t(a)$ and the observation that with probability $1-\delta$, $\Mean^*\in \Ell_t$. To obtain the second claim, apply Corollary \ref{lem:linContextual3} with $\MeanVector_*=\meanstarj, \y_t=\cxL_t(a_t), \tilde{\MeanVector}_t=[\MeanUCB_t(a_t)]_j$ (the $j^{th}$ column of $\MeanUCB_t(a_t)$), to bound $|\sum_t ([\MeanUCB_t(a_t)]_j - \meanstarj)^\top \cxL_t(a_t)| \le \sum_t |([\MeanUCB_t(a_t)]_j - \meanstarj)^\top \cxL_t(a_t)|$ for every $j$, and then take the norm.
\end{proof}	
			
%==============================================================
\section{Appendix for Section \ref{sec:coreAlgo}}
\label{app:coreAlgo}

\paragraph{Proof of Theorem \ref{th:coreAlgo}:}
We will use $\reglinCBwK$ to denote the main term in the regret  bound. 
\[\reglinCBwK(T):= O\left( m \sqrt{{\ln( m\di T/\delta)\ln(T)}{T}}\right)  \]
Let $\tau$ be the stopping time of the algorithm.  
Let $H_{t-1}$ be the history of plays and observations before time $t$, i.e. $H_{t-1} :=\{\thetaV_\tau, \cx_\tau, a_\tau,r_\tau(a_\tau), \cv_\tau(a_\tau), \tau=1,\ldots, t-1\}$. Note that $H_{t-1}$ determines $\thetaV_t, \hat{\MeanVector}_t, \MeanEst_t, \Ell_t$, but it does not determine $\cx_t, a_t, \MeanUCB_t$ (since $a_t$ and $\MeanUCB_t(a)$ depend on the context $\cx_t$ at time $t$). 
The proof is in 3 steps: 
\paragraph{Step 1:}
Since $\Ex[\cv_t(a_t) | \cx_t, a_t, H_{t-1}] = \Meanstar^\top \cxL_t(a_t)$, we apply Azuma-Hoeffding to get that with probability $1-\delta,$ 
\begin{equation}
\label{eq:vtwstar} \textstyle
 \left\|\sum_{t=1}^\tau \cv_t(a_t)-\Meanstar^\top \cxL_t(a_t)\right\|_\infty\le\reglinCBwK(T).
\end{equation}
 Similarly, a lower bound on  the sum of $\mu_*^\top \cxL_t(a_t)$ is sufficient. 

\paragraph{Step 2:} From Corollary \ref{cor:optest}, with probability $1-\delta$, 
\begin{equation}
\label{eq:whatwstar} 
\textstyle
\left\|  \sum_{t=1}^T  (\Meanstar - \MeanUCB_t(a_t))^\top\cxL_t(a_t) \right \|_\infty
\le \reglinCBwK(T).
\end{equation}
%\textcolor{red}{Nikhil: We seem to get an $m$ here instead of $\sqrt{m}$.}
It is therefore sufficient to bound  the sum of the vectors $\MeanUCB_t(a_t)^\top\cxL_t(a_t)$, 
and similarly for $\tilde{\MeanVector}_t(a_t)^\top\cxL_t(a_t)$. 
\comment{
	The proof of \prettyref{eq:whatwstar} is as follows. 
	From \prettyref{lem:actualExists},
	with probability $1-\delta$, for all $t$, $\Mean^* \in \Ell_t$, and 
	$\MeanUCB_t(a_t) \in \Ell_t$ by definition. 
	Hence  $\|\Mean^* - \MeanUCB_t(a_t)\|_{M_t} \le 2\sqrt{m\log(\frac{Tm\di}{\delta})}$. 
	(Refer to the definition of $\Ell_t$ in \eqref{eq:Ellipsoid}.) 
	Inequality \prettyref{eq:whatwstar} now follows from Lemmas \ref{lem:normTranslate}, \ref{lem:sqrtTclassic} and the above bound. 
}
We use the shorthand notation of $\rt:= \MeanVector_t(a_t)^\top\cxL_t(a_t)$, 
$\rsum := \sum_{t=1}^{\tau} \rt $, $\vt:= \MeanUCB_t(a_t)^\top\cxL_t(a_t)$ and 
$\vsum := \sum_{t=1}^{\tau} \vt $ for the rest of this proof.
\paragraph{Step 3:} 
The proof is completed by showing that  
\[ \Ex[\rsum] \geq \OPT - Z\reglinCBwK(T) . \]

\begin{lemma}\label{lem:OP1}
	\begin{equation*}
		\sum_{t=1}^{\tau} \Ex[\rt | H_{t-1}]
		\ge  \frac \tau T \OPT +Z \sum_{t=1}^{\tau} \thetaV_t \cdot 
		\Ex[\vt-{\bf 1}\frac{B}{T}|H_{t-1}] 
	\end{equation*}
\end{lemma}
\begin{proof}
	Let $\ropt_t :=\MeanVector_t(a_t)^\top\cx_t \pi^*(\cx_t) $ and 
	$\vopt_t :=\MeanUCB_t(a_t)^\top\cx_t \pi^*(\cx_t) .$ 
	By Corollary \ref{cor:optest}, 	with probability $1-\delta$, we have that 
	$T \Ex_{X_t}[\ropt_t|H_{t-1}] \geq \OPT,\text{ and } \Ex_{X_t}[\vopt_t|H_{t-1}] \le \frac {B} T {\bf 1}.$
	%Here $\thetaV_t \ge {\bf 0}$.
	By the choice made by the algorithm, 
	\begin{eqnarray*}
		\rt - Z (\thetaV_t \cdot \vt) & \ge & \ropt_t - Z (\thetaV_t \cdot \vopt_t)\\
		\Ex_{X_t}[\rt - Z (\thetaV_t \cdot \vt) | H_{t-1}] 
		& \ge &  \Ex_{X_t}[\rt | H_{t-1}] - Z (\thetaV_t \cdot \Ex[\vt | {H}_{t-1}] )\\
		& \ge & \frac 1 T \OPT  - Z \thetaV_t\cdot \frac{B{\bf 1}}{T} 
	\end{eqnarray*}
	Summing above inequality for $t=1$ to $\tau$ gives the lemma statement. 
\end{proof}

\begin{lemma}\label{lem:OP2}
	$$\sum_{t=1}^{\tau} \thetaV_t\cdot (\vt-\frac{B}{T}{\bf 1}) \ge B-\frac{\tau B}{T} - \reglinCBwK(T).$$ 
\end{lemma}
\begin{proof}
	
	Recall that $g_t(\thetaV_t) = \thetaV_t \cdot \left(\vt -\frac{B}{T}{\bf 1}\right)$, therefore the 
	LHS in the required inequality is $\sum_{t=1}^{\tau}g_t(\thetaV_t)$.  Let 
	$\thetaV^* := \arg \max_{||\thetaV||_1 \le 1, \thetaV\ge 0}\sum_{t=1}^{\tau} g_t(\thetaV) $. 
	%Using \OCO~guarantees,
	%$$\sum_{t=1}^{\tau} g_t(\thetaV_t) + \regOCO(\tau) \ge \sum_{t=1}^{\tau} g_t(\thetaV^*). $$ 
	We use the regret definition for the \OCO algorithm   to get that 
	%$$ {\regOCO(T)}   \le \epsilon\sum_{t=1}^{\tau}  g_t(\thetaV^*)   +  \frac{ \log(d+1)}{\epsilon }.$$
	$\sum_{t=1}^{\tau} g_t(\thetaV_t) \ge  \sum_{t=1}^{\tau}  g_t(\thetaV^*)-  \regOCO(T).$
	%$$= \sqrt \frac{\log(d)}{B}d(\frac{1}{T}\sum_t \vt, S)+ \frac {O(\sqrt{B\log(d)})} T.$$ 
	Note that fromt the regret bound given in Lemma \ref{lem:regOCO}, $\regOCO(T) \leq \reglinCBwK(T)$. 
\paragraph{Case 1: $\tau < T$.}	This means that $\sum_{t=1}^\tau (\cv_t (a_t)\cdot {\bf e}_j)\geq B$ for some $j$.
 Then from (\ref{eq:vtwstar}) and (\ref{eq:whatwstar}), it must be that $\sum_{t=1}^\tau (\vt \cdot {\bf e}_j)\geq B -\reglinCBwK(T)$  so that $\sum_{t=1}^{\tau} g_t(\thetaV^*) \ge \sum_{t=1}^{\tau} g_t({\bf e}_j) \geq B-\frac{\tau B}{T} -\reglinCBwK(T)$. 

\paragraph{Case 2: $\tau = T$.} In this case,  $B - \frac{\tau}{T}B = 0$
$=  \sum_{t=1}^{\tau}  g_t({\mathbf 0})$
	$ \leq \sum_{t=1}^{\tau}  g_t(\thetaV^*),$ 
	which completes the proof of the lemma. 
\end{proof}

%--------------------------------------------------------------------------------
Now, we are ready to prove Theorem \ref{th:coreAlgo}, which states that Algorithm \ref{algo:coreAlgo} achieves a regret of $Z\reglinCBwK(T)$.
%\begin{theorem}
%	For the online packing problem,  Algorithm \ref{algo:packing} achieves a CR of $1-O(\epsilon)$ where $\epsilon = \sqrt{\tfrac{\log(d)}{B}}$. 
%\end{theorem}
\noindent {\it \bf Proof of Theorem \ref{th:coreAlgo}.} 
Substituting the inequality from Lemma \ref{lem:OP2} in Lemma \ref{lem:OP1}, we get 
\begin{eqnarray*}
	\sum_{t=1}^{\tau} \Ex[\rt| {H}_{t-1}]	%& \ge & \frac{\tau}{T} \OPT_{sum} + Z \sum_{t=1}^{\tau} \Ex[\thetaV_t \cdot (\vt - \frac{B}{T}{\bf 1})| {H}_{t-1}] 	- \sum_{t=1}^\tau {\cal Q}(t)\\
	%& \ge & Z B + \frac{\tau}{T} \OPT_{sum} -Z \frac{\tau B}{T} - \sum_{t=1}^\tau {\cal Q}(t) - {\cal R}(\tau)\\
	& \ge & \frac \tau T \OPT+ ZB \left(1-\frac{\tau}{T}\right) - Z \reglinCBwK(T)
\end{eqnarray*}

Also, $Z\ge \frac{\OPT}{B}$. Substituting in above,
\begin{eqnarray*}
	\Ex[\rsum] = \sum_{t=1}^{\tau} \Ex[\rt| {H}_{t-1}]	
	& \ge & \frac \tau T \OPT +  \OPT( 1 -  \frac \tau T)  -Z \regOCO(T) \\
	& \ge & \OPT -Z \reglinCBwK(T)  
\end{eqnarray*}

From Steps 1 and 2, this implies a lower bound on $\Ex[\sum_{t=1}^\tau r_t(a_t)]$. 
The proof is now completed by using Azuma-Hoeffding to bound the actual total reward with high probability. 
\qed

%=========================================================================
\section{Appendix for Section \ref{sec:packing}}
\comment{
\begin{lemma}
\label{lem:estimatingOPT}
For $\gamma=\gammaValue$, with probability $1-O(\delta)$,
 \begin{center} $ \OPT - L\gamma \le \OPTest^{2\gamma} \le \OPT+6\gamma(Z^*+L).$ \end{center}
\end{lemma}
For \linCbudget, $Z^*\le \frac{\OPT}{(B/T)}, L=1, \oneNorm=\oneNorm_{\infty}=1$. Therefore, $Z=\frac{\OPTest^{2\gamma} + L\gamma}{(B/T)} +1$ satisfies the bounds of Lemma \ref{lem:packingZest}. 
}
\begin{proof}[\bf Proof of Lemma \ref{lem:estimatingOPT}]

Let us define an ``intermediate sample optimal" as:
\begin{eqnarray}
\label{eq:OPTinterim}
\OPTinterim^\gamma & := & \begin{array}{rcl}
 \max_{q} & \frac{T}{T_0}\sum_{i=1}^{T_0} \MeanVector_*^\top \cx_i \pi(\cx_i)])& \\
\text{such that} &  \frac{T}{T_0}\sum_{i=1}^{T_0}\Mean_*^\top \cx_i \pi(\cx_i) \le B + \gamma& 
\end{array}
\end{eqnarray}
Above sample optimal knows the parameters $\MeanVector_*, \Mean_*$, the error comes only from approximating the expected value over context distribution by average over the observed contexts.  We do not actually compute $\OPTinterim^\gamma$, but will use it for the convenience of proof exposition. The proof involves two steps. 

%First, we prove that $\|\sum_{i=1}^{T_0} (\hat{\Mean}_0 - \Mean^*)\cx_i q(\cx_i)\|$ is small for all $q$. Then, we  bound $\|\frac{1}{T_0} \sum_{i=1}^{T_0} \Mean^* \cx_i q(\cx_i) - \Ex[\Mean^* \cx q(\cx)] \|$ for all $q$.
%\paragraph{Step 1: Bound $\|\sum_{i=1}^{T_0} (\MeanEst_0 - \Mean_0')\cx_i q(\cx_i)\|$.} 

\begin{enumerate}
\item[Step 1:] Bound $|\OPTinterim^\gamma-\OPT|$.
\item[Step 2:] Bound $|\OPTest^{2\gamma}- \OPTinterim^\gamma|$ 
%by bounding $|\sum_{i=1}^{T_0}(\hat{\MeanVector}_i- \MeanVector_*)^\top \cx_i \pi(\cx_i)|, \|\sum_{i=1}^{T_0} (\hat{\Mean}_i - \Mean_*)^\top\cx_i \pi(\cx_i)\|$ for all $\pi$.
\end{enumerate}

{\bf Step 1} bound can be borrowed from the work on Online Stochastic Convex Programming in \cite{AD15}: since $\MeanVector_*, \Mean^*$ is known, so there is effectively full information before making the decision, i.e., consider the vectors $[\MeanVector_*^\top \cxL_t(a), \Mean_*^\top\cxL_t(a)]$ as outcome vectors  which can be observed for all arms $a$ {\it before} choosing the distribution over arms to be played at time $t$, therefore, the setting in \cite{AD15} applies. In fact, $\hat{\OPT}^\gamma$ as defined by Equation (F.10) in \cite{AD15} when $A_t = \{[\MeanVector_*^\top \cxL_t(a), \Mean_*^\top\cxL_t(a)], a\in [K]\}$, $f$ identity, and $S=\{\cv_{-1} \le \frac{B}{T}\}$, is same as $\frac{1}{T}$ times  $\OPTinterim^\gamma$ defined here. And using Lemma F.4 and Lemma F.6 in \cite{AD15} (using $L=1, Z^*=\OPT/B$), we obtain that for any $\gamma \ge \left(\frac{T}{T_0}\right) \gammaValue$, with probability $1-O(\delta)$,
\begin{equation}
\label{eq:intermediateClaim}
 \OPT -\gamma \le \OPTinterim^\gamma \le \OPT + 2\gamma (\frac{\OPT}{B}+1).
\end{equation}

For {\bf Step 2}, we show that with probability $1-\delta$, for {\em all} $\pi$, $\gamma \ge \left(\frac{T}{T_0}\right) \gammaValue$
\begin{equation}
\label{eq:neededClaim2.0}
|\sum_{i=1}^{T_0}(\hat{\MeanVector}_i- \MeanVector_*)^\top \cx_i \pi(\cx_i)| \le \gamma
\end{equation}
\begin{equation}
\label{eq:neededClaim2}
\|\frac{T}{T_0}\sum_{i=1}^{T_0} (\MeanEst_i - \Mean_*)^\top\cx_i \pi(\cx_i)\|_\infty \le \gamma
\end{equation}
This is sufficient to prove both lower and upper bound on $\OPTest^{2\gamma}$ for $\gamma \ge \left(\frac{T}{T_0}\right) \gammaValue$. For lower bound, we can simply use \eqref{eq:neededClaim2} for optimal policy for $\OPTinterim^\gamma$, denoted by $\bar{\pi}$. This implies that (because of relaxation of distance constraint by $\gamma$) $\bar{\pi}$ is a feasible primal solution for $\OPTest^{2\gamma}$, and therefore using \eqref{eq:intermediateClaim} and \eqref{eq:neededClaim2.0},
$$\OPTest^{2\gamma} +\gamma \ge \OPTinterim^\gamma \ge \OPT-\gamma.$$
For the upper bound, we can use \eqref{eq:neededClaim2} for the optimal policy $\hat{\pi}$ for $\OPTest^{2\gamma}$. Then, using \eqref{eq:intermediateClaim} and \eqref{eq:neededClaim2.0},
$$\OPTest^{2\gamma} \le \OPTinterim^{3\gamma} + \gamma \le \OPT + 6\gamma(\frac{\OPT}{B}+1) + \gamma.$$

Combining, this proves the desired lemma statement:
\begin{equation}
\OPT - 2\gamma \le \OPTest^{2\gamma} \le \OPT+7\gamma(\frac{\OPT}{B}+1)
\end{equation}
%, using the primal and dual formulations of convex program in \eqref{eq:OPTestProgram}, respectively. 
%the claim follows from Lipschitz property of $f$, on applying above claim to the optimal $q$ for $\OPT$ and $\OPT^\gamma$.

What remains is to proof the claim in \eqref{eq:neededClaim2.0} and \eqref{eq:neededClaim2}. We show the proof for \eqref{eq:neededClaim2}, the proof for \eqref{eq:neededClaim2.0} is similar. Observe that for any $\pi$,
\begin{eqnarray*}
 \|\sum_{t=1}^{T_0} (\MeanEst_t - \Mean_*)^\top \cx_t \pi(\cx_t)\|_\infty & \le & \sum_{t=1}^{T_0} \|({\MeanEst}_t-\Mean_*)^\top \cx_t \pi(\cx_t)\|_\infty \\
& \le & \sum_{t=1}^{T_0} \|\MeanEst_t-\Mean_*\|_{M_t} \|\cx_t \pi(\cx_t)\|_{M_t^{-1}} 
\end{eqnarray*}
where $\|\MeanEst_t-\Mean_*\|_{M_t} = \max_j \|\meanhattj-\meanstarj\|_{M_t}$.

Now, applying Lemma \ref{lem:linContextual1} to every column $\meanhattj$ of $\MeanEst_t$, we have that with probability $1-\delta$ for all $t$, 
$$ \|\MeanEst_t-\Mean_*\|_{M_t} \le 2\sqrt{m\log(td/\delta)} \le 2\sqrt{m\log(T_0d/\delta)}$$
And, by choice of $p_t$
$$\|\cx_t \pi(\cx_t)\|_{M_t^{-1}} \le \|\cx_t p_t\|_{M_t^{-1}}.$$
Also,  by Lemma \ref{lem:linContextual2},
$$\sum_{t=1}^{T_0} \|\cx_t p_t\|_{M_t^{-1}} \le \sqrt{mT_0 \ln(T_0)}$$
Therefore, substituting,
\begin{eqnarray*}
 \|\sum_{t=1}^{T_0} (\MeanEst_t -\Mean_*)^\top \cx_t \pi(\cx_t)\|_\infty & \le &  (2\sqrt{m\log(T_0d/\delta)})\sum_{t=1}^{T_0} \|\cx_t p_t\|_{M_t^{-1}} \\
& \le & (2\sqrt{m\log(T_0d/\delta)}) \sqrt{mT_0\ln(T_0)} \\
& \le & \frac{T_0}{T}\gamma
\end{eqnarray*}

%Then, to complete the proof of Claim in Equation \eqref{eq:neededClaim2}, we show that $\|\sum_{i=1}^{T_0} \Mean^* \Ex[\cx q(\cx)] - \Mean^*\cx_i q(\cx_i)\| \le \gamma$. 

%This completes the proof of Claim in Equation \eqref{eq:neededClaim2}.

\end{proof}

\end{document}